\pdfoutput=1
\documentclass[11pt]{article}

\PassOptionsToPackage{table}{xcolor}
\usepackage[preprint]{acl}

\usepackage{times}
\usepackage{latexsym}
\usepackage{booktabs}
\usepackage[table]{xcolor}
\usepackage{multirow}
\usepackage{graphicx}

\usepackage[T1]{fontenc}

\usepackage[utf8]{inputenc}

\usepackage{microtype}

\usepackage{inconsolata}

\usepackage{graphicx}
\usepackage{amsmath, amssymb, amsthm}
\usepackage{algorithm}
\usepackage{algpseudocode} %
\usepackage{booktabs}
\usepackage{subcaption}
\usepackage[most]{tcolorbox}

\newcommand{\boldhdr}[1]{
    \vspace{0.1cm}
    \noindent\textbf{#1}.
}
\newtheorem{theorem}{Theorem}[section]
\newtheorem{proposition}[theorem]{Proposition}

\usepackage{enumitem}

\newcommand{\rateinline}[2]{#1\color{gray}{\tiny$\pm$ #2}}

\title{Rethinking the Role of Temperature in Large Language Model Distillation}

\author{Hoang-Chau Luong \ \ \ \ \  Lingwei Chen \\
  Golisano College of Computing and Information Sciences \\
  Rochester Institute of Technology \\
  Rochester, NY, United States \\
  \texttt{cl6300@rit.edu, lwcics@rit.edu}}

\begin{document}
\maketitle

\begin{abstract}
    Reverse Kullback–Leibler (RKL) divergence is widely favored over forward KL (FKL) in large language models (LLM) distillation, yet this preference is largely based on comparisons that omit the temperature $\tau$, overlooking its central role in softening teacher distributions and improving knowledge transfer. In this work, we revisit temperature in LLM distillation and show that it fundamentally changes the comparison between FKL and RKL. Our analysis reveals an asymmetric effect: temperature substantially enriches FKL with non-dominant token signals, whereas it mainly rescales RKL gradients, causing FKL to benefit much more from $\tau$ scaling than RKL. This asymmetry overturns the standard empirical conclusion: although RKL outperforms FKL at $\tau=1$, FKL consistently surpasses RKL at higher temperatures across instruction-following benchmarks. Moreover, the impact of temperature is not limited to FKL; it improves a broader family of distillation objectives, enabling simple KL-based methods to achieve competitive performance against recent state-of-the-art LLM distillation approaches. 
\end{abstract}

\section{Introduction}
\label{sec:introduction}
\vspace{-0.1cm}

Knowledge distillation (KD)~\citep{hinton2015distilling} is a widely used paradigm for model compression that transfers knowledge from a large teacher model to a smaller student model~\citep{Romero15-iclr, cho2019efficacy, gou2021knowledge}. A central component of classical KD is the distillation temperature $\tau$, which softens teacher distribution and reveals informative ``dark knowledge'' beyond the top prediction~\citep{hinton2015distilling, tang2020understanding, zhao2022decoupled}. By redistributing probability mass from dominant
classes to lower-probability alternatives, KD temperature prevents the distillation signal from collapsing onto the teacher's top prediction, and provides the student with richer inter-class relational information.

Although temperature has been extensively studied in vision-based KD~\cite{beyer2022knowledge,li2023ctkd,sun2024logit,sun2025knowledge}, it is often omitted in LLM distillation. Early sequence-level KD work for language tasks~\citep{kim2016sequence} reported that $\tau=1$ achieved the best empirical performance and therefore adopted the default softmax distribution without further analyzing the role of temperature. Subsequent LLM distillation methods similarly formulate and evaluate their distillation frameworks under the default $\tau=1$ setting~\citep{gu2023minillm, agarwal2024policy, ko2024distillm, wang2025abkd, song2026survey}, implicitly removing temperature from the design space of KL-based distillation objectives.

This convention directly affects how recent work compares forward KL (FKL) and reverse KL (RKL) objectives, where the commonly reported superiority of RKL is largely established through evaluations conducted at $\tau=1$~\citep{gu2023minillm, ko2024distillm, wu2025rethinking}. Unlike RKL, whose optimization is weighted by the student distribution, FKL directly weights each sample loss by the teacher probabilities. As a result, low-entropy teacher soft targets suppress the non-target probability mass that FKL relies on, making it appear inferior not due to an inherent limitation of the objective itself, but because it is evaluated under insufficiently softened teacher distributions. This raises an important question: \emph{Is RKL intrinsically superior for LLM distillation or has FKL been systematically underestimated?}

\vspace{-0.1cm}
\subsection{Related Work}
\label{sec:related}
\vspace{-0.1cm}

\boldhdr{Temperature in KD}
Given teacher logits $z^t$ and student logits $z^s$, temperature-scaled distributions are defined as
\begin{equation}
\label{eq:prob_definition}
    p^\tau = \mathrm{softmax}(z^t / \tau), 
    \
    q^\tau = \mathrm{softmax}(z^s / \tau),
\end{equation}
where $\tau \geq 1$. Increasing $\tau$ smooths the distribution by reducing probability mass on dominant tokens and amplifying lower-probability alternatives. This smoothing has been widely used in vision-based distillation~\citep{hinton2015distilling, zhao2022decoupled, li2023ctkd, jin2023multi, sun2024logit, cui2024decoupled, wei2024scaled}, as well as in early NLP distillation methods such as Seq-KD~\citep{kim2016sequence}, DistilBERT~\citep{sanh2019distilbert}, and TinyBERT~\citep{jiao2020tinybert}. In LLM distillation, however, most methods omit KD temperature and evaluate objectives at $\tau=1$~\citep{gu2023minillm, ko2024distillm, wu2025rethinking, song2026survey}. Although recent work has revisited temperature for FKL~\citep{song2026survey} or proposed adaptive temperature strategies~\citep{xie2026llm, luong2026consistently}, its objective-dependent effect on FKL and RKL remains poorly understood.

\boldhdr{KL objectives for LLM distillation}
FKL and RKL are two standard distribution-matching objectives. FKL minimizes $D_{\mathrm{KL}}(p^\tau \Vert q^\tau)$, encouraging the student to cover the teacher distribution, while RKL minimizes $D_{\mathrm{KL}}(q^\tau \Vert p^\tau)$ and is associated with mode-seeking behavior. Recent LLM distillation methods favor RKL due to its strong empirical performance~\citep{gu2023minillm, ko2024distillm}, leading to variants such as AKL~\citep{wu2025rethinking}, AB-KD~\citep{wang2025abkd}, SFKL/SRKL~\citep{ko2024distillm}, and DRKL~\citep{luong2026diversity}. However, these comparisons are typically conducted under $\tau=1$ setting, leaving the interaction between temperature and KL objective design underexplored.

\vspace{-0.1cm}
\subsection{Contributions}
\label{sec:contributions}
\vspace{-0.1cm}

We revisit the role of temperature in LLM distillation and address three key questions:

\vspace{0.1cm}\noindent \textbf{RQ1: How does temperature distinguish FKL from RKL?}
We theoretically show that both FKL and RKL converge to the same logit-matching behavior in the high-temperature regime. At practical temperatures, KD temperature reshapes the teacher soft targets and enriches the non-target knowledge transfer for FKL, while primarily rescaling RKL gradients. Thus, KD temperature changes \emph{what} FKL learns, but mainly changes \emph{how} RKL learns.

\vspace{0.1cm}\noindent \textbf{RQ2: When does FKL outperform RKL?}
We demonstrate that the FKL--RKL comparison is highly temperature-dependent. While RKL often performs better at $\tau=1$, FKL benefits substantially more from softened teacher targets and consistently surpasses RKL at higher KD temperatures. This finding challenges the common view that RKL is intrinsically superior for LLM distillation.

\vspace{0.1cm}\noindent \textbf{RQ3: Does temperature broadly improve KL-based distillation objectives?}
We further show that KD temperature improves not only FKL, but also a broader family of KL-based objectives, including Sym-KL, JS, SFKL, and AKL. With appropriate KD temperature, these simple objectives achieve performance competitive with existing LLM distillation methods. This suggests that KD temperature is not merely a correction for FKL, but a general objective-dependent factor for designing and fairly comparing LLM distillation methods.

\begin{figure*}[t]
\vspace{-2mm}
    \centering
    
    \begin{subfigure}[t]{0.32\linewidth}
        \centering
        \includegraphics[width=\linewidth]{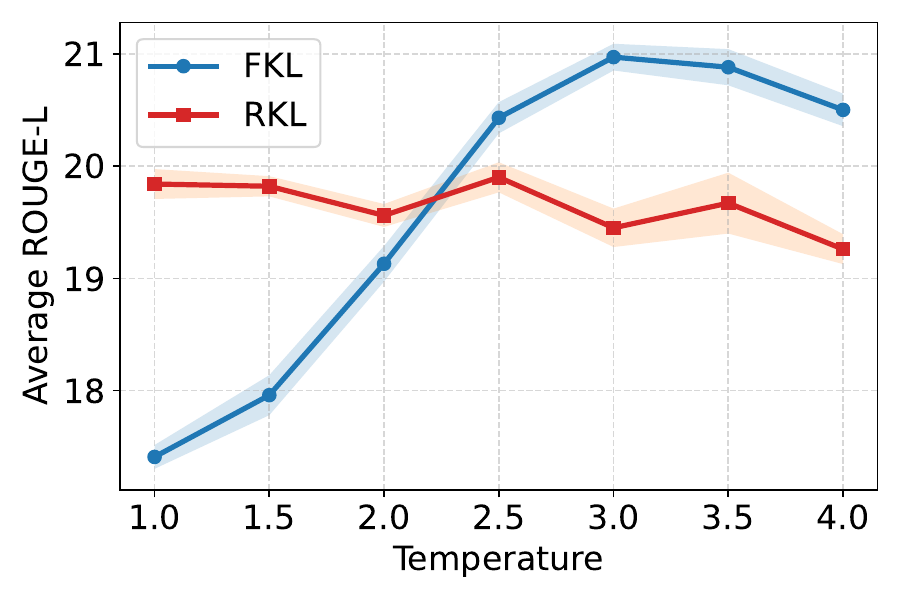}
        \vspace{-7mm}
        \caption{GPT-2 XL $\to$ GPT-2 Base}
        \label{fig:base}
    \end{subfigure}
    \hfill
    \begin{subfigure}[t]{0.32\linewidth}
        \centering
        \includegraphics[width=\linewidth]{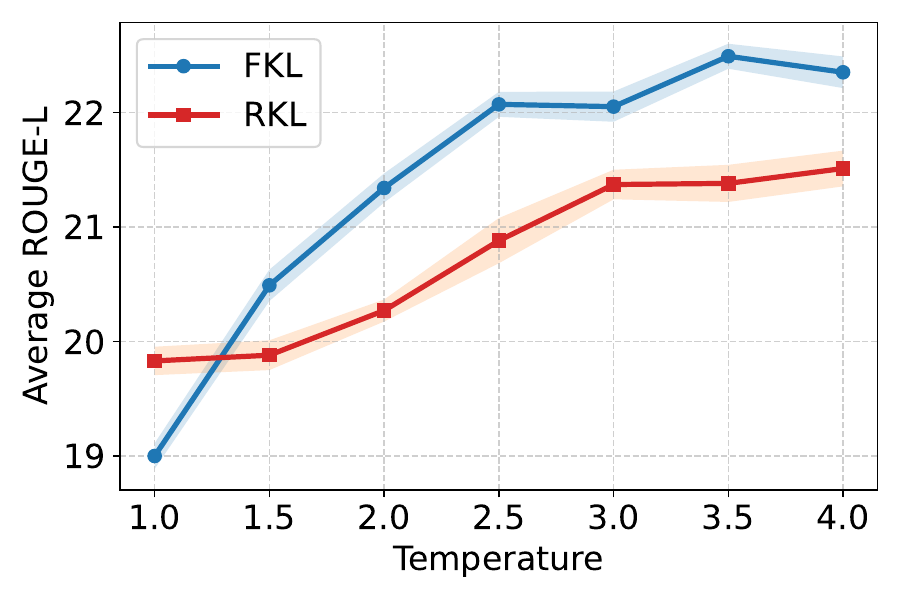}
        \vspace{-7mm}
        \caption{GPT-2 XL $\to$ GPT-2 Medium}
        \label{fig:medium}
    \end{subfigure}
    \hfill
    \begin{subfigure}[t]{0.32\linewidth}
        \centering
        \includegraphics[width=\linewidth]{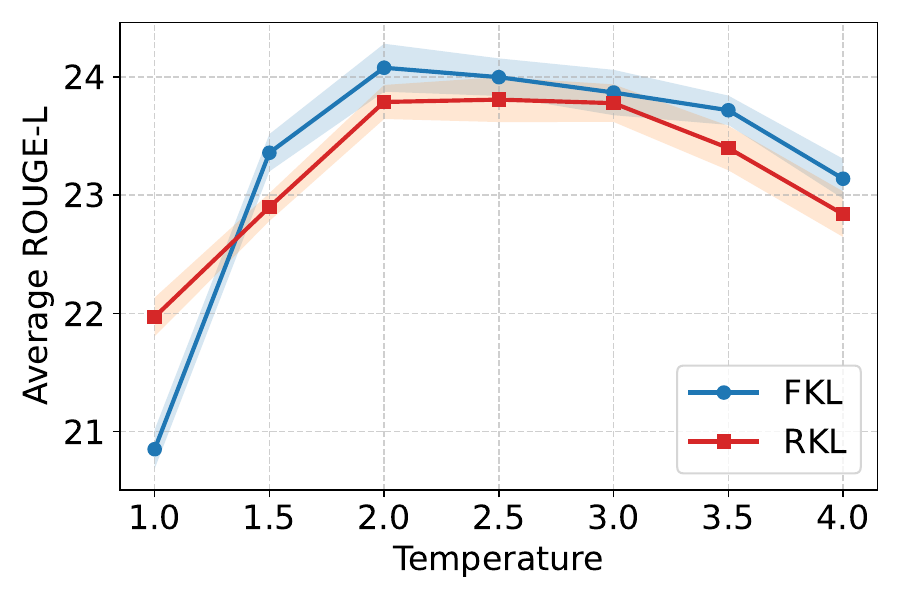}
        \vspace{-7mm}
        \caption{OPT 6.7B $\to$ OPT 1.3B}
        \label{fig:opt}
    \end{subfigure}

    \vspace{-3mm}
    \caption{Effect of temperature on FKL and RKL across model scales. Temperature consistently improves FKL while providing limited benefit for RKL, leading to a reversal in their performance at higher temperatures.
    }    
    \vspace{-5mm}
    \label{fig:combined}
\end{figure*}

\vspace{-0.1cm}
\section{How Does Temperature Distinguish FKL from RKL?}
\label{sec:rq1}
\vspace{-0.1cm}

This section theoretically analyzes the behavior of FKL and RKL under KD temperature, which shows that the two objectives become equivalent in the high-temperature regime, but diverge substantially at low temperatures, explaining why FKL can outperform RKL in LLM
distillation.

Let $z^t,z^s\in\mathbb{R}^{V}$ denote the teacher and student logits over a vocabulary of size $V$. The temperature-scaled teacher and student distributions, denoted by $p^\tau$ and $q^\tau$, are defined in Eq.~\eqref{eq:prob_definition}. We study $\mathcal{L}_{\mathrm{FKL}}=D_{\mathrm{KL}}(p^\tau\Vert q^\tau) =
    \sum_{i=1}^{V} p_i^\tau \log (p_i^\tau / q_i^\tau)$ and $\mathcal{L}_{\mathrm{RKL}}=D_{\mathrm{KL}}(q^\tau\Vert p^\tau)$.
Their gradients with respect to student logit $z_i^s$ are
\begin{align}
\label{eq:fkl_gradients}
    \nabla_{z_i^s}\mathcal{L}_{\mathrm{FKL}}
    &=
    \frac{1}{\tau}(q_i^\tau-p_i^\tau),
    \\
    \nabla_{z_i^s}\mathcal{L}_{\mathrm{RKL}}
    &=
    \frac{1}{\tau}q_i^\tau
    \left(
    \log \frac{q_i^\tau}{p_i^\tau} -\mathcal{L}_{\mathrm{RKL}}
    \right).
    \label{eq:rkl_gradients}
\end{align}

\begin{proposition}[]
\label{prop:high_temp}
Assume the teacher and student logits are centered, i.e., $\sum_i z_i^t=\sum_i z_i^s=0$. 
As $\tau\to\infty$, FKL and RKL gradients satisfy
\begin{equation*}
    \nabla_{z_i^s}\mathcal{L}_{\mathrm{FKL}}
    =
    \nabla_{z_i^s}\mathcal{L}_{\mathrm{RKL}}
    =
    \frac{1}{V\tau^2}(z_i^s-z_i^t)
    +O(\tau^{-3}).
\end{equation*}
\end{proposition}

\boldhdr{High-temperature regime: FKL and RKL become logit matching}
Recent studies~\citep{wu2025rethinking, luong2026diversity} suggest that FKL and RKL reach the same solution under idealized assumptions, such as sufficient student capacity and exact global optimization. Proposition~\ref{prop:high_temp} (proof is provided in Appendix~\ref{appendix:proposition1}) shows that temperature alone is sufficient to recover this equivalence: as $\tau\to\infty$, both objectives share the same gradient and reduce to teacher--student logit matching. Thus, the practical distinction between FKL and RKL must come from their low-temperature regime.

\boldhdr{Low-temperature regime: temperature separates FKL from RKL}
At practical temperatures, FKL directly matches the student distribution to the softened teacher distribution, whereas RKL still operates through student-weighted teacher--student logit gaps. For FKL, $p^\tau$ serves as the probability target, and changing $\tau$ changes the supervision itself:
$
    \frac{\partial p_i^\tau}{\partial \tau}
    =
    \frac{p_i^\tau}{\tau^2}
    (\mathbb{E}_{j\sim p^\tau}[z_j^t]-z_i^t).
$
Increasing temperature reduces the dominance of high-logit tokens and redistributes probability mass to lower-logit tokens. 
Together with Eq.~\eqref{eq:fkl_gradients}, this shows temperature changes the strength of the FKL update, and teacher distribution the student is trained to match.

\begin{proposition}[]
\label{prop:finite_temp}
Let $\Delta_i=z_i^s-z_i^t$ denote the student--teacher logit gap at index $i$, the RKL gradient can be rewritten as
\begin{equation}
\label{eq:rkl_gap_gradient}
    \nabla_{z_i^s}\mathcal{L}_{\mathrm{RKL}}
    =
    \frac{q_i^\tau}{\tau^2}
    \left(
    \Delta_i-\mathbb{E}_{j\sim q^\tau}[\Delta_j]
    \right),
\end{equation}
\end{proposition}

\vspace{-0.1cm}
\noindent For RKL, temperature plays a different role. 
Proposition~\ref{prop:finite_temp} (see Appendix~\ref{appendix:proposition2} for the detailed proof) shows that the RKL gradient is driven by teacher--student logit gaps rather than direct probability matching between student and softened teacher distributions as in FKL. Thus, temperature mainly conditions RKL optimization by rescaling the gradient by $1/\tau^2$ and smoothing the student-side weights $q_i^\tau$. \citet{hinton2015distilling} counteract this rescaling by multiplying the gradient by $\tau^2$, while the smoothing effect of $q_i^\tau$ still remains.

\begin{tcolorbox}[colback=blue!4,colframe=blue!40,left=1mm,right=1mm,top=1mm,bottom=1mm, before skip=4pt, after skip=3pt]
\textbf{Insight.}
Temperature changes \emph{what} FKL learns, but mainly \emph{how} RKL learns. 
At low-regime temperatures, FKL receives richer non-target supervision from softened teacher distribution, while RKL gets gradient rescaling.
\end{tcolorbox}

\vspace{-0.1cm}
\section{When Does FKL Outperform RKL?}
\label{sec:rq2}
\vspace{-0.1cm}

Our analysis leads to a natural question: \emph{Can temperature change which KL objective is better?} We answer this question empirically on five instruction-following benchmarks including Dolly Eval, Self-Instruct~\citep{wang2023self}, Vicuna Eval~\citep{vicuna2023}, Super-Natural Instructions (Super-NI)~\citep{wang2022super}, and Unnatural Instructions (UnNI)~\citep{honovich2023unnatural} using GPT-2~\citep{radford2019language} and OPT models~\citep{zhang2022opt}, with implementation details provided in Appendix~\ref{app:implementation}. Figure~\ref{fig:combined} presents the comparative average ROUGE-L across five datasets, and reveals a clear reversal. Under $\tau=1$, RKL outperforms FKL, consistent with prior LLM distillation results. However, once temperature is introduced, FKL consistently surpasses RKL, showing that its apparent weakness at $\tau=1$ stems from insufficient access to softened non-target teacher signals. In particular, we draw the following insights.

\begin{table}[t]
\centering
\scriptsize
\setlength{\tabcolsep}{3.0pt}
\renewcommand{\arraystretch}{0.98}
\caption{Distillation under on-policy setting for GPT-2 XL $\to$ GPT-2 Base. We report ROUGE-L mean and standard deviation over five random seeds. Bold: Avg.\ Rouge-L exceeds the default ($\tau=1.0$).}
\label{tab:on_policy}
\vspace{-0.2cm}
\resizebox{\linewidth}{!}{%
\begin{tabular}{llcccccc}
\toprule
\textbf{Loss} & \textbf{$\tau$} & \textbf{Dolly} & \textbf{SelfInst} & \textbf{Super-NI} & \textbf{UnNI} & \textbf{Vicuna} & \textbf{Avg.} \\
\midrule
\multirow{7}{*}{FKL}
 & 1.0 & \rateinline{22.03}{0.86} & \rateinline{10.29}{0.40} & \rateinline{17.42}{0.18} & \rateinline{20.11}{0.10} & \rateinline{14.89}{0.43} & 16.95 \\
 & 1.5 & \rateinline{23.62}{0.49} & \rateinline{10.57}{0.55} & \rateinline{19.01}{0.08} & \rateinline{21.40}{0.09} & \rateinline{14.27}{0.40} & \textbf{17.77} \\
 & 2.0 & \rateinline{24.16}{0.11} & \rateinline{11.81}{0.62} & \rateinline{20.96}{0.05} & \rateinline{23.64}{0.14} & \rateinline{15.20}{0.43} & \textbf{19.15} \\
 & 2.5 & \rateinline{26.00}{0.54} & \rateinline{12.62}{0.54} & \rateinline{22.35}{0.34} & \rateinline{25.83}{0.10} & \rateinline{16.06}{0.56} & \textbf{20.57} \\
 & 3.0 & \rateinline{25.30}{0.30} & \rateinline{12.11}{0.25} & \rateinline{23.70}{0.20} & \rateinline{26.79}{0.26} & \rateinline{16.62}{0.58} & \textbf{20.90} \\
 & 3.5 & \rateinline{25.32}{0.07} & \rateinline{12.40}{0.22} & \rateinline{24.92}{0.44} & \rateinline{27.20}{0.02} & \rateinline{16.35}{0.29} & \textbf{21.24} \\
 & 4.0 & \rateinline{25.60}{0.48} & \rateinline{12.36}{0.44} & \rateinline{24.28}{0.26} & \rateinline{26.63}{0.09} & \rateinline{16.42}{0.50} & \textbf{21.06} \\
\midrule
\multirow{7}{*}{RKL}
 & 1.0 & \rateinline{24.12}{0.25} & \rateinline{12.20}{0.38} & \rateinline{22.43}{0.09} & \rateinline{24.93}{0.06} & \rateinline{16.14}{0.37} & 19.97 \\
 & 1.5 & \rateinline{23.96}{0.13} & \rateinline{12.00}{0.62} & \rateinline{23.16}{0.26} & \rateinline{26.78}{0.16} & \rateinline{16.14}{0.37} & \textbf{20.41} \\
 & 2.0 & \rateinline{24.56}{0.45} & \rateinline{12.33}{0.09} & \rateinline{22.76}{0.18} & \rateinline{25.86}{0.15} & \rateinline{15.64}{0.72} & \textbf{20.23} \\
 & 2.5 & \rateinline{24.48}{0.13} & \rateinline{10.96}{0.46} & \rateinline{21.88}{0.22} & \rateinline{24.94}{0.01} & \rateinline{15.90}{0.44} & 19.63 \\
 & 3.0 & \rateinline{24.36}{0.37} & \rateinline{11.06}{0.13} & \rateinline{19.92}{0.48} & \rateinline{23.63}{0.13} & \rateinline{16.01}{0.12} & 19.00 \\
 & 3.5 & \rateinline{24.00}{0.18} & \rateinline{10.66}{0.14} & \rateinline{20.41}{0.38} & \rateinline{23.48}{0.22} & \rateinline{15.78}{0.30} & 18.87 \\
 & 4.0 & \rateinline{23.85}{0.31} & \rateinline{11.26}{0.03} & \rateinline{19.88}{0.43} & \rateinline{23.15}{0.15} & \rateinline{16.07}{0.39} & 18.84 \\
\bottomrule
\end{tabular}}
\label{table:on_policy}
\vspace{-0.4cm}
\end{table}

\boldhdr{Temperature turns FKL from weak to outperforming RKL}
Increasing temperature softens the distribution, exposes richer teacher information, and yields large gains. For GPT-2 Base, FKL improves from $19.41$ to $20.97$ average ROUGE-L at $\tau=3.0$. Similar gains appear for larger students: GPT-2 Medium rises from around $19$ to above $22$, and OPT-1.3B from below $21$ to above $24$, with best results around $\tau=2.5$--$3.5$. Across model scales, these gains reverse the ranking: FKL shifts from trailing RKL at $\tau=1$ to outperforming it at proper temperatures. 
For example, FKL exceeds RKL by $1.5$ points at $\tau=3.0$ on GPT-2 Base, by over $1.0$ point at $\tau=3.5$ on GPT-2 Medium, and by $0.27$ points at $\tau=2.0$ on OPT 1.3B.

\boldhdr{RKL benefits less from temperature}
For GPT-2 Base, RKL remains nearly flat across temperatures and can even degrade at larger $\tau$. For larger students, the gains are modest, about $1$ point on GPT-2 Medium and $2$ points on OPT-1.3B, compared with consistent gains above $3$ points for FKL. This agrees with our theoretical analysis: temperature mainly smooths and reconditions the RKL gradient without giving RKL direct access to the softened teacher distribution as a probability target.

\boldhdr{On-policy distillation}
To examine whether our findings extend beyond the off-policy setting, we run on-policy GPT-2 XL $\to$ GPT-2 Base distillation with both FKL and RKL (Table~\ref{table:on_policy}). The same asymmetric effect of temperature persists: FKL improves from 16.95 at $\tau=1$ to 21.24 at $\tau=3.5$ ($+4.29$ average Rouge-L), whereas RKL improves by only $+0.44$ before degrading at higher $\tau$. The relative ranking also reverses: although RKL outperforms FKL at $\tau=1$, FKL surpasses RKL at an appropriate KD temperature. These results show that our finding is not specific to the off-policy setting, but holds under on-policy distillation.

\begin{tcolorbox}[colback=blue!4,colframe=blue!40,left=1mm,right=1mm,top=1mm,bottom=1mm, before skip=4pt, after skip=0pt]
\textbf{Insight.}
    FKL outperforms RKL when temperature exposes useful non-target teacher information. 
    At $\tau=1$, FKL underuses the teacher because the target is too sharp. 
    At higher temperature, the softened distribution becomes sufficiently informative for FKL to surpass RKL.
\end{tcolorbox}

\vspace{-0.1cm}
\section{Does Temperature Broadly Improve KL-based Distillation Objectives?}
\label{sec:rq3}
\vspace{-0.1cm}

\definecolor{taurow}{RGB}{245,248,255}
\newcommand{\posgain}[1]{\textcolor{blue}{$\uparrow$#1}}
\newcommand{\secbest}[1]{\underline{#1}}

\begin{table}[t]
    \centering
    \scriptsize
    \setlength{\tabcolsep}{3.0pt}
    \renewcommand{\arraystretch}{0.98}
    \caption{
    Effect of temperature scaling on distillation losses.
    We report ROUGE-L mean and standard deviation over five random seeds.
    }
    \vspace{-2mm}
    \resizebox{1.0\linewidth}{!}{
    \begin{tabular}{lccccc}
        \toprule
        \textbf{Method}
        & \textbf{Dolly}
        & \textbf{Super-NI}
        & \textbf{UnNI}
        & \textbf{Avg.}
        & \textbf{$\Delta$ Avg.} \\
        \midrule

        \multicolumn{6}{l}{\textbf{\textit{GPT-2 XL (1.5B) $\to$ GPT-2 Base (0.1B)}}} \\
        \midrule
        GPT-2 XL
        & \rateinline{27.00}{0.19}
        & \rateinline{26.46}{0.41}
        & \rateinline{31.10}{0.06}
        & 28.19
        & -- \\
        \midrule

        FKL
        & \rateinline{23.80}{0.55}
        & \rateinline{16.27}{0.24}
        & \rateinline{19.03}{0.09}
        & 19.70
        & -- \\

        \rowcolor{taurow}
        FKL + $\tau$
        & \rateinline{\textbf{25.83}}{0.24}
        & \rateinline{23.78}{0.26}
        & \rateinline{\secbest{26.97}}{0.12}
        & \secbest{25.53}
        & \secbest{\posgain{5.83}} \\

        RKL
        & \rateinline{24.67}{0.13}
        & \rateinline{21.03}{0.21}
        & \rateinline{23.94}{0.12}
        & 23.21
        & -- \\

        \rowcolor{taurow}
        RKL + $\tau$
        & \rateinline{25.00}{0.24}
        & \rateinline{21.41}{0.34}
        & \rateinline{24.68}{0.07}
        & 23.70
        & \posgain{0.49} \\

        Sym-KL
        & \rateinline{24.38}{0.14}
        & \rateinline{19.92}{0.25}
        & \rateinline{22.21}{0.03}
        & 22.17
        & -- \\

        \rowcolor{taurow}
        Sym-KL + $\tau$
        & \rateinline{\secbest{25.54}}{0.14}
        & \rateinline{21.69}{0.10}
        & \rateinline{25.29}{0.14}
        & 24.17
        & \posgain{2.00} \\

        JS
        & \rateinline{24.08}{0.26}
        & \rateinline{19.66}{0.23}
        & \rateinline{22.17}{0.21}
        & 21.97
        & -- \\

        \rowcolor{taurow}
        JS + $\tau$
        & \rateinline{25.51}{0.40}
        & \rateinline{\secbest{24.34}}{0.16}
        & \rateinline{27.03}{0.11}
        & \textbf{25.63}
        & \posgain{3.66} \\

        SFKL
        & \rateinline{24.44}{0.42}
        & \rateinline{22.01}{0.23}
        & \rateinline{23.47}{0.12}
        & 23.31
        & -- \\

        \rowcolor{taurow}
        SFKL + $\tau$
        & \rateinline{25.33}{0.28}
        & \rateinline{\textbf{24.40}}{0.27}
        & \rateinline{\textbf{27.07}}{0.16}
        & \secbest{25.60}
        & \posgain{2.29} \\

        SRKL
        & \rateinline{24.48}{0.37}
        & \rateinline{23.25}{0.36}
        & \rateinline{24.01}{0.05}
        & 23.91
        & -- \\

        \rowcolor{taurow}
        SRKL + $\tau$
        & \rateinline{25.03}{0.42}
        & \rateinline{23.95}{0.14}
        & \rateinline{25.13}{0.08}
        & 24.70
        & \posgain{0.79} \\

        AKL
        & \rateinline{21.83}{0.19}
        & \rateinline{15.40}{0.20}
        & \rateinline{18.06}{0.13}
        & 18.43
        & -- \\

        \rowcolor{taurow}
        AKL + $\tau$
        & \rateinline{25.35}{0.27}
        & \rateinline{23.63}{0.16}
        & \rateinline{26.18}{0.16}
        & 25.05
        & \textbf{\posgain{6.62}} \\

        \midrule
        AB
        & \rateinline{24.32}{0.29}
        & \rateinline{23.08}{0.19}
        & \rateinline{24.32}{0.14}
        & 23.91
        & -- \\

        DRKL
        & \rateinline{25.51}{0.42}
        & \rateinline{23.39}{0.11}
        & \rateinline{26.57}{0.11}
        & 25.16
        & -- \\

        \midrule
        \midrule

        \multicolumn{6}{l}{\textbf{\textit{OPT 6.7B $\to$ OPT 1.3B}}} \\
        \midrule
        OPT 6.7B
        & \rateinline{27.52}{0.29}
        & \rateinline{30.41}{0.46}
        & \rateinline{31.39}{0.20}
        & 29.77
        & -- \\
        \midrule

        FKL
        & \rateinline{26.07}{0.65}
        & \rateinline{22.11}{0.38}
        & \rateinline{26.82}{0.10}
        & 25.00
        & -- \\

        \rowcolor{taurow}
        FKL + $\tau$
        & \rateinline{\secbest{28.20}}{0.48}
        & \rateinline{27.62}{0.45}
        & \rateinline{31.43}{0.17}
        & 29.08
        & \posgain{4.08} \\

        RKL
        & \rateinline{26.58}{0.11}
        & \rateinline{22.64}{0.16}
        & \rateinline{26.29}{0.10}
        & 25.17
        & -- \\

        \rowcolor{taurow}
        RKL + $\tau$
        & \rateinline{27.64}{0.48}
        & \rateinline{26.96}{0.27}
        & \rateinline{31.29}{0.13}
        & 28.63
        & \posgain{3.46} \\

        Sym-KL
        & \rateinline{25.73}{0.40}
        & \rateinline{23.40}{0.22}
        & \rateinline{25.44}{0.07}
        & 24.86
        & -- \\

        \rowcolor{taurow}
        Sym-KL + $\tau$
        & \rateinline{27.74}{0.40}
        & \rateinline{\secbest{28.21}}{0.16}
        & \rateinline{31.73}{0.15}
        & 29.23
        & \posgain{4.37} \\

        JS
        & \rateinline{26.39}{0.71}
        & \rateinline{24.02}{0.31}
        & \rateinline{26.59}{0.11}
        & 25.67
        & -- \\

        \rowcolor{taurow}
        JS + $\tau$
        & \rateinline{27.97}{0.39}
        & \rateinline{\textbf{28.27}}{0.26}
        & \rateinline{\secbest{32.20}}{0.09}
        & \secbest{29.48}
        & \posgain{3.81} \\

        SFKL
        & \rateinline{25.98}{0.44}
        & \rateinline{23.62}{0.19}
        & \rateinline{26.15}{0.11}
        & 25.25
        & -- \\

        \rowcolor{taurow}
        SFKL + $\tau$
        & \rateinline{27.52}{0.64}
        & \rateinline{27.39}{0.20}
        & \rateinline{32.03}{0.13}
        & 28.98
        & \posgain{3.73} \\

        SRKL
        & \rateinline{26.04}{0.36}
        & \rateinline{22.96}{0.22}
        & \rateinline{24.71}{0.22}
        & 24.57
        & -- \\

        \rowcolor{taurow}
        SRKL + $\tau$
        & \rateinline{27.83}{0.38}
        & \rateinline{28.20}{0.37}
        & \rateinline{31.68}{0.14}
        & 29.24
        & \secbest{\posgain{4.67}} \\

        AKL
        & \rateinline{26.20}{0.27}
        & \rateinline{22.12}{0.18}
        & \rateinline{25.89}{0.20}
        & 24.74
        & -- \\

        \rowcolor{taurow}
        AKL + $\tau$
        & \rateinline{\textbf{28.35}}{0.19}
        & \rateinline{28.10}{0.23}
        & \rateinline{\textbf{32.40}}{0.15}
        & \textbf{29.62}
        & \textbf{\posgain{4.88}} \\

        \midrule
        AB
        & \rateinline{26.86}{0.17}
        & \rateinline{24.39}{0.11}
        & \rateinline{26.62}{0.09}
        & 25.96
        & -- \\

        DRKL
        & \rateinline{27.72}{0.36}
        & \rateinline{25.56}{0.36}
        & \rateinline{29.84}{0.15}
        & 27.71
        & -- \\

        \bottomrule
    \end{tabular}
    }
    \vspace{-4mm}
    \label{tab:main1}
\end{table}

We now ask whether temperature is broadly beneficial for KL-based distillation objectives.
To answer this, we evaluate FKL, RKL, and Sym-KL, together with recent objectives such as JS~\citep{agarwal2024policy}, SFKL/SRKL~\citep{ko2024distillm}, AKL~\citep{wu2025rethinking}, AB~\citep{wang2025abkd}, and DRKL~\citep{luong2026diversity}. 
Experiments are conducted on three instruction-following benchmarks including Dolly, Super-NI, and UnNI, with two additional benchmarks reported in Appendix~\ref{app:more_exp} and implementation detailed in Appendix~\ref{app:implementation}.

\begin{figure}[t]
    \centering
    \includegraphics[width=\linewidth]{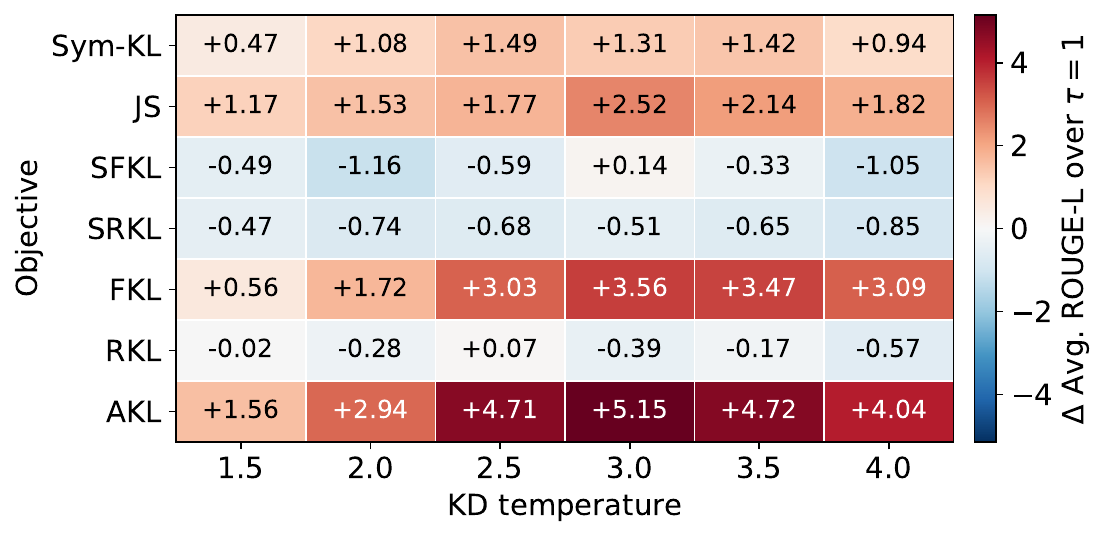}
    \vspace{-9mm}
    \caption{
    Temperature-sensitivity heatmap on GPT-2 XL $\rightarrow$ GPT-2 Base.
    }
    \label{fig:temp_gain_heatmap}
    \vspace{-2mm}
\end{figure}

\boldhdr{Temperature broadly improves KD objectives}
Table~\ref{tab:main1} compares each objective at default $\tau=1$ with its temperature-tuned variant. Temperature improves all objectives across GPT-2 and OPT models, although the gains vary substantially across objectives. For GPT-2 Base, FKL improves from $19.70$ to $25.53$ average ROUGE-L ($+5.83$), while AKL achieves the largest gain ($+6.62$). JS, SFKL, and Sym-KL also improve notably, whereas RKL and SRKL gain only $+0.49$ and $+0.79$, supporting our analysis that temperature is especially beneficial for FKL-objectives that directly exploit softened teacher probabilities. For OPT-1.3B, the gains become broadly substantial, ranging from $+3.46$ for RKL to $+4.88$ for AKL. Figure~\ref{fig:temp_gain_heatmap} visualizes ROUGE-L gains across temperatures and objectives. FKL-based objectives (FKL, Sym-KL, JS, and AKL) exhibit consistent gains, while RKL-based objectives (RKL, SRKL, and SFKL) are less stable and can even degrade. Thus, temperature is more robustly beneficial for FKL-based objectives but requires careful tuning for RKL-based ones.

\boldhdr{Temperature changes objective rankings}
Applying KD temperature substantially strengthens standard KL-based objectives and changes their relative ranking against stronger baselines. On GPT-2 Base, JS+$\tau$ achieves the best average score ($25.63$), while FKL+$\tau$ reaches $25.53$, surpassing state-of-the-art AB ($23.91$) and DRKL ($25.16$). On OPT-1.3B, FKL+$\tau$ improves to $29.08$, outperforming both AB ($25.96$) and DRKL ($27.71$), while AKL+$\tau$ achieves the best overall score ($29.62$). Thus, temperature is not a minor implementation detail: it can change which objective appears strongest and is necessary for fair comparison.

\begin{table}[t]
    \centering
    \scriptsize
    \setlength{\tabcolsep}{3.0pt}
    \renewcommand{\arraystretch}{0.98}
    \caption{
    Pointwise LLM-as-a-judge 1-10 score of best-$\tau$
    vs. $\tau=1$ on OPT 6.7B $\to$ OPT 1.3B, per method
    (Claude Haiku 4.5, MT-Bench single-answer rubric, score $\times$10 shown), across
    all five instruction-following evaluation sets.
    }
    \vspace{-2mm}
    \resizebox{1.0\linewidth}{!}{
    \begin{tabular}{lccccccc}
        \toprule
        \textbf{Method}
        & \textbf{Dolly}
        & \textbf{Vicuna}
        & \textbf{Self-Inst}
        & \textbf{Super-NI}
        & \textbf{UnNI}
        & \textbf{Avg.}
        & \textbf{$\Delta$Avg.} \\
        \midrule
        FKL
        & 40.3 & 23.0 & 22.7 & 16.0 & \secbest{26.8}
        & 25.8 & -- \\
        \rowcolor{taurow}
        FKL + $\tau$
        & \textbf{44.3} & \textbf{29.0} & \secbest{27.7} & \secbest{20.2} & \textbf{29.2}
        & \textbf{30.1} & +4.3 \\
        RKL
        & 36.7 & 27.3 & 25.0 & 19.6 & 17.5
        & 25.2 & -- \\
        \rowcolor{taurow}
        RKL + $\tau$
        & 37.3 & 27.6 & 26.2 & \secbest{20.2} & 19.0
        & 26.1 & +0.9 \\
        Sym-KL
        & 38.0 & 20.3 & 26.0 & 16.4 & 17.2
        & 23.6 & -- \\
        \rowcolor{taurow}
        Sym-KL + $\tau$
        & 38.7 & 27.0 & 26.0 & 16.4 & 18.4
        & 25.3 & +1.7 \\
        SFKL
        & 34.0 & 26.3 & 23.7 & 15.6 & 21.6
        & 24.2 & -- \\
        \rowcolor{taurow}
        SFKL + $\tau$
        & 39.7 & \secbest{28.7} & 24.0 & 17.6 & 22.4
        & 26.5 & +2.3 \\
        SRKL
        & 37.3 & 23.7 & 26.0 & 16.8 & 18.8
        & 24.5 & -- \\
        \rowcolor{taurow}
        SRKL + $\tau$
        & 40.3 & 27.7 & 26.3 & 18.4 & 19.2
        & 26.4 & +1.9 \\
        AKL
        & 40.7 & 23.7 & 23.7 & 16.4 & 18.0
        & 24.5 & -- \\
        \rowcolor{taurow}
        AKL + $\tau$
        & \secbest{42.0} & 25.0 & \textbf{28.0} & \textbf{21.2} & 21.6
        & \secbest{27.6} & +3.1 \\
        \bottomrule
    \end{tabular}
    }
    \vspace{-4mm}
    \label{tab:laaj_opt_5ds}
\end{table}

\boldhdr{LLM-as-a-judge}
We additionally conduct an LLM-as-a-judge evaluation on OPT 6.7B $\to$ OPT 1.3B across five instruction-following benchmarks, using Claude Haiku 4.5 with the MT-Bench pointwise and pairwise protocols~\cite{zheng2023judging}. Results are reported in Tables~\ref{tab:laaj_opt_5ds} and~\ref{tab:laaj_winrate}. This evaluation independently corroborates the ROUGE-L findings. Pointwise evaluation shows that FKL$+\tau$ achieves the highest average score (30.1 vs.\ 25.8 at $\tau=1$, $+4.3$), while RKL$+\tau$ improves by only $+0.9$. Pairwise evaluation reaches a similar conclusion: FKL$+\tau$ wins 57.7\% of head-to-head comparisons against its $\tau=1$ counterpart, whereas RKL$+\tau$ wins only 42.3\%, indicating that temperature tuning does not consistently benefit RKL. The same trend extends to other FKL-based objectives (Sym-KL, SFKL, and AKL), all of which are consistently preferred after temperature tuning. These results provide complementary evidence that the asymmetric effect of KD temperature is not specific to ROUGE-L, further strengthening that KD temperature benefits FKL-based objectives substantially more than RKL-based objectives.

\begin{table}[t]
    \centering
    \scriptsize
    \setlength{\tabcolsep}{3.0pt}
    \renewcommand{\arraystretch}{0.98}
    \caption{Pairwise LLM-as-a-judge win rates for best-$\tau$ vs.\ $\tau=1$ on OPT 6.7B $\to$ OPT 1.3B. Evaluation uses Claude Haiku 4.5, MT-Bench pairwise rubric with position-swap bias control. We report $\tau$-adjusted win rate: $\text{win\%} + 0.5\times\text{tie\%}$, where 50\% indicates parity, $>$50\% favors the temperature-tuned variant, and $<$50\% prefers $\tau=1$.}
    \vspace{-2mm}
    \resizebox{1.0\linewidth}{!}{
    \begin{tabular}{lcccccc}
        \toprule
        \textbf{Method}
        & \textbf{Dolly}
        & \textbf{Vicuna}
        & \textbf{Self-Inst}
        & \textbf{Super-NI}
        & \textbf{UnNI}
        & \textbf{Avg.} \\
        \midrule
        FKL
        & 60.0 & 60.0 & 61.7 & 53.3 & 53.3
        & 57.7 \\
        RKL
        & 45.0 & 43.3 & 43.3 & 33.3 & 46.7
        & 42.3 \\
        Sym-KL
        & 51.7 & 66.7 & 53.3 & 53.3 & 53.3
        & 55.7 \\
        SFKL
        & 56.7 & 68.3 & 51.7 & 53.3 & 58.3
        & 57.7 \\
        SRKL
        & 55.0 & 51.7 & 56.7 & 53.3 & 60.0
        & 55.3 \\
        AKL
        & 51.7 & 53.3 & 66.7 & 65.0 & 60.0
        & 59.3 \\
        \bottomrule
    \end{tabular}
    }
    \vspace{-4mm}
    \label{tab:laaj_winrate}
\end{table}

\boldhdr{Modern model families} 
To further strengthen the empirical evidence on modern model families, we additionally evaluated two teacher–student pairs: Llama 3.2-3B~\cite{grattafiori2024llama} $\to$ Llama 3.2-1B and Qwen3-4B~\cite{yang2025qwen3technicalreport} $\to$ Qwen3-1.7B. The results are reported in Appendix~\ref{app:more_exp}. These results indicate that the asymmetric effect of temperature is not specific to any particular LLM architecture.

\boldhdr{Case study}
Appendix~\ref{app:case_studies} provides case studies on the Unnatural Instructions test set, comparing model outputs under different objectives.

\begin{tcolorbox}[colback=blue!4,colframe=blue!40,left=1mm,right=1mm,top=1mm,bottom=1mm, before skip=4pt, after skip=0pt]
\textbf{Insight.}
Temperature broadly improves KL-based distillation. 
Its gains are largest for FKL-based objectives, but weaker for RKL-based objectives. Thus, KD temperature should be considered essential for fair comparison, rather than a minor implementation detail.
\end{tcolorbox}

\section{Conclusion}

In this work, we revisit the role of KD temperature in LLM distillation and show that it can change the relative behavior of FKL and RKL. 
KD temperature makes non-target teacher probabilities more informative, allowing FKL to transfer richer knowledge and surpass RKL despite underperforming at $\tau=1$.
Beyond that, KD temperature also improves a broad range of KL-based distillation objectives, showing that it is essential for fair comparison.

\section*{Acknowledgments}
This work is partially supported by the NSF under grant CNS-2549995. The authors also acknowledge Research Computing at the Rochester Institute of Technology for providing computational resources and support that have contributed to the research results reported in this publication.

\section*{Limitations}

Due to computational constraints, our language experiments use teacher models with up to 7B parameters, and evaluating larger teachers remains future work. For the Qwen3 and Llama3.2 teacher--student pairs, we adopt LoRA rather than full fine-tuning. This follows common practice for teacher models at this scale and enables us to evaluate a broader range of model families and sizes within a fixed computational budget. As with most work in this setting, our evaluation centers on English-language instruction following. Extending it to other languages, task types, and training distributions is a natural next step.

\bibliography{custom}

\newpage

\appendix

This appendix provides additional implementation details and supporting theoretical and empirical materials for the main paper:

\begin{itemize}[leftmargin=*,itemsep=1pt,topsep=2pt]
    \item Appendix~\ref{app:implementation} presents implementation, training, and evaluation details, including datasets, teacher--student settings, compared objectives, and temperature tuning protocols.
    
    \item Appendix~\ref{appendix:proposition1} provides proof for Proposition~\ref{prop:high_temp}.

    \item Appendix~\ref{appendix:proposition2} provides proof for Proposition~\ref{prop:finite_temp}.
    
    \item Appendix~\ref{app:more_exp} reports additional experimental results on the Self-Instruct and Vicuna benchmarks across all objectives, as well as results on the modern model families Llama3.2 and Qwen3.
    
    \item Appendix~\ref{app:case_studies} presents qualitative case studies illustrating how KD temperature improves instruction following and output quality under different distillation objectives.

    \item Appendix~\ref{appx:llms} describes the limited use of AI assistants for language refinement during manuscript preparation.
\end{itemize}

\section{Implementation Details}
\label{app:implementation}

The implementation is publicly available at \url{https://github.com/lhchau/rethinking-kd-temperature-llm}.

\boldhdr{Training resources}
All experiments were conducted on a GPU cluster with four NVIDIA A100 GPUs, each with 40GB memory.

\boldhdr{Experimental setup}
We evaluate LLM distillation in an off-policy instruction-following setting, where student models are trained on fixed teacher-generated responses. 
Following~\citet{gu2023minillm}, we use the instruction--response dataset constructed from \texttt{databricks-dolly-15k}~\citep{DatabricksBlog2023DollyV2}, containing 14k training examples, 500 validation examples, and 500 test examples. 
We first fine-tune the teacher models on this dataset and then distill their responses into smaller students. 
Our teacher--student pairs include GPT-2 XL~\citep{radford2019language} (1.5B) $\rightarrow$ GPT-2 Base (120M) and OPT 6.7B~\citep{zhang2022opt} $\rightarrow$ OPT 1.3B.

\boldhdr{Training protocol}
To isolate the effect of the distillation objective and KD temperature, all methods are trained under the same protocol. 
For both GPT-2 and OPT families, we use a batch size of 32, train for 20 epochs, and set the maximum input length to 512 tokens. 
Following~\citet{gu2023minillm}, the learning rate is set to $5\times10^{-4}$ for GPT-2 Base and $5\times10^{-5}$ for OPT-1.3B. 
All compared methods use the same training data, teacher checkpoint, student initialization, and hyperparameter budget, ensuring that performance differences are attributable to the distillation objective and temperature design rather than implementation or training discrepancies.

\boldhdr{Evaluation protocol}
We evaluate the distilled students on five instruction-following benchmarks, including Dolly Eval, Self-Instruct~\citep{wang2023self}, Vicuna Eval~\citep{vicuna2023}, Super-Natural Instructions (Super-NI)~\citep{wang2022super}, and Unnatural Instructions (UnNI)~\citep{honovich2023unnatural}. 
We report ROUGE-L~\citep{lin2004rouge} averaged over five random seeds $\{10,20,30,40,50\}$. 
Model checkpoints are saved after each epoch, and final results are reported using the checkpoint with the best validation ROUGE-L. 
The decoding temperature is set to 1 during evaluation.

\boldhdr{Compared objectives}
We compare KL-family and recent LLM distillation objectives under the same off-policy setting. 
The KL-family baselines include FKL, RKL~\citep{gu2023minillm}, symmetric KL (Sym-KL), Jensen--Shannon divergence (JS)~\citep{agarwal2024policy}, skewed FKL (SFKL)~\citep{ko2024distillm}, skewed RKL (SRKL)~\citep{ko2024distillm}. 
For Sym-KL, we use $0.5\,\mathrm{FKL}+0.5\,\mathrm{RKL}$ and for SFKL and SRKL, we set the smoothing parameter to $\alpha=0.1$. 
We also compare with stronger recent objectives, including Adaptive KL (AKL)~\citep{wu2025rethinking}, $\alpha$-$\beta$ divergence (AB)~\citep{wang2025abkd} with $(\alpha=0.2,\beta=0.7)$, and diversity-aware RKL (DRKL)~\citep{luong2026diversity}. 
Since our focus is off-policy distillation from fixed teacher responses, we exclude methods that require on-policy sampling or additional external datasets. 
For all baselines, we follow the hyperparameter settings from the original papers and use the implementation provided by DistilLM~\citep{ko2024distillm}.

For the temperature-tuned variants reported in Tables~\ref{tab:main1} and~\ref{tab:extra}, including FKL+$\tau$, RKL+$\tau$, Sym-KL+$\tau$, JS+$\tau$, SFKL+$\tau$, SRKL+$\tau$, and AKL+$\tau$, we apply KD temperature on top of each original distillation objective while keeping its implementation and default hyperparameters unchanged. 
The only tuned factor is the KD temperature, for which we perform a grid search over $\tau \in \{1.5, 2.0, 2.5, 3.0, 3.5, 4.0\}$ and report the best-performing setting for comparison. 
Thus, the gains reported in Tables~\ref{tab:main1} and~\ref{tab:extra} are attributable solely to KD temperature, not to additional objective-specific hyperparameter tuning. 
Further tuning of each objective's own hyperparameters could potentially yield additional improvements.

\section{Proof for Proposition~\ref{prop:high_temp}}\label{appendix:proposition1}

\begin{proof}
For large $\tau$, we use the first-order expansion
$\exp(z_i/\tau)=1+z_i/\tau+O(\tau^{-2})$. 
For the teacher distribution, this gives
\begin{align*}
    \sum_j \exp(z_j^t/\tau)
    &=
    V+\frac{1}{\tau}\sum_j z_j^t+O(\tau^{-2}) \\
    &=
    V+O(\tau^{-2}),
\end{align*}
where the last equality follows from the centered-logit assumption 
$\sum_j z_j^t=0$. 
Thus,
\begin{equation*}
    p_i^\tau
    =
    \frac{1+z_i^t/\tau+O(\tau^{-2})}
         {V+O(\tau^{-2})}
    =
    \frac{1}{V}
    +
    \frac{z_i^t}{V\tau}
    +
    O(\tau^{-2}).
\end{equation*}
Applying the same argument to the student logits gives
\[
    q_i^\tau
    =
    \frac{1}{V}
    +
    \frac{z_i^s}{V\tau}
    +
    O(\tau^{-2}).
\]
\noindent \textbf{For FKL part.} Substituting these into the FKL gradient in Eq.~\eqref{eq:fkl_gradients} gives
$$
\nabla_{z_i^s}\mathcal{L}_{\mathrm{FKL}}
=
\frac{1}{V\tau^2}(z_i^s-z_i^t)+O(\tau^{-3}).
$$

\noindent \textbf{For RKL part.} Since $q_j^\tau=1/V+O(\tau^{-1})$ and the centered-logit assumption gives 
$\sum_j\Delta_j=0$, we have
\begin{align*}
\mathbb{E}_{j\sim q^\tau}[\Delta_j]
&=
\sum_j q_j^\tau\Delta_j \\
&=
\frac{1}{V}\sum_j\Delta_j
+
\sum_j O(\tau^{-1})\Delta_j \\
&=
O(\tau^{-1}).
\end{align*}
Using this, Proposition~\ref{prop:finite_temp}, and $q_i^\tau=1/V+O(\tau^{-1})$ as $\tau \to \infty$, we have
$$
\nabla_{z_i^s}\mathcal{L}_{\mathrm{RKL}}
=
\frac{1}{V\tau^2}(z_i^s-z_i^t)+O(\tau^{-3}).
$$
This completes the proof.
\end{proof}

\section{Proof for Proposition~\ref{prop:finite_temp}}\label{appendix:proposition2}

\begin{proof} 
The log-ratio can be decomposed as
\begin{equation}
\label{eq:rkl_log_ratio}
    \log\frac{q_i^\tau}{p_i^\tau}
    =
    \frac{\Delta_i}{\tau}
    +
    C_\tau,
    \text{where }
    C_\tau=\log\frac{Z_t^\tau}{Z_s^\tau},
\end{equation}
$Z_t^\tau=\sum_j\exp(z_j^t/\tau)$ and $Z_s^\tau=\sum_j\exp(z_j^s/\tau)$. The term $C_\tau$ comes only from the softmax normalization constants and is independent of token index $i$.
From the definition of $\mathcal{L}_{\mathrm{RKL}}$, we have
\[
    \mathcal{L}_{\mathrm{RKL}}
    =
    \mathbb{E}_{j\sim q^\tau}
    \left[
    \log\frac{q_j^\tau}{p_j^\tau}
    \right]
    =
    \frac{1}{\tau}\mathbb{E}_{j\sim q^\tau}[\Delta_j]
    +
    C_\tau .
\] 
Taking Eq.~\eqref{eq:rkl_log_ratio} minus this $\mathcal{L}_{\mathrm{RKL}}$, we obtain
\[
    \log\frac{q_i^\tau}{p_i^\tau}
    -
    \mathcal{L}_{\mathrm{RKL}}
    =
    \frac{1}{\tau}
    \left(
    \Delta_i-\mathbb{E}_{j\sim q^\tau}[\Delta_j]
    \right).
\]
Substituting this into Eq.~\eqref{eq:rkl_gradients} gives
\[
    \nabla_{z_i^s}\mathcal{L}_{\mathrm{RKL}}
    =
    \frac{q_i^\tau}{\tau^2}
    \left(
    \Delta_i-\mathbb{E}_{j\sim q^\tau}[\Delta_j]
    \right).
\]
This completes the proof.
\end{proof}

\definecolor{taurow}{RGB}{245,248,255}

\begin{table}[t]
    \centering
    \scriptsize
    \setlength{\tabcolsep}{3.0pt}
    \renewcommand{\arraystretch}{0.98}
    \caption{
    Effect of temperature scaling on distillation losses.
    We report ROUGE-L mean and standard deviation over five random seeds.
    Avg. is computed over Self-Inst. and Vicuna.
    }
    \resizebox{1.0\linewidth}{!}{
    \begin{tabular}{lcccc}
        \toprule
        \textbf{Method}
        & \textbf{Self-Instruct}
        & \textbf{Vicuna}
        & \textbf{Avg.} $(\uparrow)$
        & \textbf{$\Delta$ Avg.} \\
        \midrule

        \multicolumn{5}{l}{\textbf{\textit{GPT-2 XL (1.5B) $\to$ GPT-2 Base (0.1B)}}} \\
        \midrule
        GPT-2 XL
        & \rateinline{14.07}{0.37}
        & \rateinline{16.31}{0.32}
        & 15.19
        & -- \\
        \midrule

        FKL
        & \rateinline{9.68}{0.49}
        & \rateinline{14.53}{0.31}
        & 12.11
        & -- \\

        \rowcolor{taurow}
        FKL + $\tau$
        & \rateinline{12.16}{0.42}
        & \rateinline{\textbf{16.71}}{0.17}
        & \secbest{14.44}
        & \secbest{\posgain{2.33}} \\

        RKL
        & \rateinline{11.25}{0.35}
        & \rateinline{15.82}{0.43}
        & 13.54
        & -- \\

        \rowcolor{taurow}
        RKL + $\tau$
        & \rateinline{11.97}{0.36}
        & \rateinline{16.46}{0.38}
        & 14.22
        & \posgain{0.68} \\

        Sym-KL
        & \rateinline{11.07}{0.43}
        & \rateinline{15.72}{0.75}
        & 13.40
        & -- \\

        \rowcolor{taurow}
        Sym-KL + $\tau$
        & \rateinline{11.88}{0.24}
        & \rateinline{16.35}{0.37}
        & 14.12
        & \posgain{0.72} \\

        JS
        & \rateinline{11.67}{0.33}
        & \rateinline{15.47}{0.37}
        & 13.57
        & -- \\

        \rowcolor{taurow}
        JS + $\tau$
        & \rateinline{\textbf{12.71}}{0.37}
        & \rateinline{16.05}{0.29}
        & 14.38
        & \posgain{0.81} \\

        SFKL
        & \rateinline{11.57}{0.15}
        & \rateinline{15.35}{0.30}
        & 13.46
        & -- \\

        \rowcolor{taurow}
        SFKL + $\tau$
        & \rateinline{11.77}{0.26}
        & \rateinline{\secbest{16.68}}{0.59}
        & 14.23
        & \posgain{0.77} \\

        SRKL
        & \rateinline{10.69}{0.21}
        & \rateinline{14.96}{0.34}
        & 12.83
        & -- \\

        \rowcolor{taurow}
        SRKL + $\tau$
        & \rateinline{11.17}{0.40}
        & \rateinline{15.42}{0.22}
        & 13.30
        & \posgain{0.47} \\

        AKL
        & \rateinline{9.57}{0.24}
        & \rateinline{13.70}{0.21}
        & 11.64
        & -- \\

        \rowcolor{taurow}
        AKL + $\tau$
        & \rateinline{\secbest{12.64}}{0.13}
        & \rateinline{16.52}{0.34}
        & \textbf{14.58}
        & \textbf{\posgain{2.94}} \\

        \midrule
        AB
        & \rateinline{11.05}{0.21}
        & \rateinline{15.82}{0.58}
        & 13.44
        & -- \\

        DRKL
        & \rateinline{12.34}{0.58}
        & \rateinline{16.33}{0.31}
        & 14.34
        & -- \\

        \midrule
        \midrule

        \multicolumn{5}{l}{\textbf{\textit{OPT 6.7B $\to$ OPT 1.3B}}} \\
        \midrule
        OPT 6.7B
        & \rateinline{16.42}{0.69}
        & \rateinline{17.64}{0.27}
        & 17.03
        & -- \\
        \midrule

        FKL
        & \rateinline{12.84}{0.32}
        & \rateinline{16.71}{0.33}
        & 14.78
        & -- \\

        \rowcolor{taurow}
        FKL + $\tau$
        & \rateinline{15.21}{0.49}
        & \rateinline{17.95}{0.57}
        & 16.58
        & \secbest{\posgain{1.80}} \\

        RKL
        & \rateinline{12.29}{0.74}
        & \rateinline{17.54}{0.13}
        & 14.92
        & -- \\

        \rowcolor{taurow}
        RKL + $\tau$
        & \rateinline{15.03}{0.40}
        & \rateinline{\secbest{18.11}}{0.65}
        & 16.57
        & \posgain{1.66} \\

        Sym-KL
        & \rateinline{13.56}{0.68}
        & \rateinline{16.95}{0.37}
        & 15.26
        & -- \\

        \rowcolor{taurow}
        Sym-KL + $\tau$
        & \rateinline{14.94}{1.01}
        & \rateinline{\textbf{18.14}}{0.69}
        & 16.54
        & \posgain{1.29} \\

        JS
        & \rateinline{12.63}{0.51}
        & \rateinline{17.15}{0.39}
        & 14.89
        & -- \\

        \rowcolor{taurow}
        JS + $\tau$
        & \rateinline{14.93}{0.73}
        & \rateinline{18.05}{0.37}
        & 16.49
        & \posgain{1.60} \\

        SFKL
        & \rateinline{11.95}{0.84}
        & \rateinline{16.87}{0.45}
        & 14.41
        & -- \\

        \rowcolor{taurow}
        SFKL + $\tau$
        & \rateinline{14.50}{0.41}
        & \rateinline{17.87}{0.46}
        & 16.19
        & \posgain{1.78} \\

        SRKL
        & \rateinline{11.99}{0.28}
        & \rateinline{17.34}{0.25}
        & 14.67
        & -- \\

        \rowcolor{taurow}
        SRKL + $\tau$
        & \rateinline{\textbf{15.41}}{0.45}
        & \rateinline{18.04}{0.49}
        & \textbf{16.73}
        & \textbf{\posgain{2.06}} \\

        AKL
        & \rateinline{13.11}{0.31}
        & \rateinline{17.26}{0.66}
        & 15.19
        & -- \\

        \rowcolor{taurow}
        AKL + $\tau$
        & \rateinline{\secbest{15.27}}{0.79}
        & \rateinline{18.10}{0.24}
        & \secbest{16.69}
        & \posgain{1.50} \\

        \midrule
        AB
        & \rateinline{12.93}{0.47}
        & \rateinline{17.59}{0.41}
        & 15.26
        & -- \\

        DRKL
        & \rateinline{14.37}{0.31}
        & \rateinline{18.04}{0.24}
        & 16.21
        & -- \\

        \bottomrule
    \end{tabular}
    }
    \vspace{-4mm}
    \label{tab:extra}
\end{table}

\begin{table*}[t]
\centering
\scriptsize
\caption{Rouge-L ($\uparrow$) for distillation from Llama3.2-3B to Llama3.2-1B using FKL, RKL, AKL, SFKL, and SRKL. $\Delta$ denotes the change in average Rouge-L from applying temperature tuning. Bold: best value per column.}
\label{tab:temperature_1B3B}
\resizebox{0.85\linewidth}{!}{%
\begin{tabular}{lccccccc}
\toprule
Loss & Dolly & Self-Inst & Super-NI & UnNI & Vicuna & Avg. & $\Delta$ \\
\midrule
FKL                  & 25.86{\scriptsize$\pm$0.54} & 15.46{\scriptsize$\pm$1.03} & 28.14{\scriptsize$\pm$0.25} & 29.52{\scriptsize$\pm$0.13} & 15.53{\scriptsize$\pm$0.04} & 22.90 & -- \\
\rowcolor{taurow}
FKL$+\tau$           & 27.34{\scriptsize$\pm$0.23} & 18.73{\scriptsize$\pm$0.24} & 33.04{\scriptsize$\pm$0.29} & 34.07{\scriptsize$\pm$0.14} & 17.11{\scriptsize$\pm$0.66} & 26.06 & $+3.16$ \\
\midrule
RKL                  & 28.85{\scriptsize$\pm$0.23} & 18.90{\scriptsize$\pm$0.10} & \textbf{34.94}{\scriptsize$\pm$0.15} & \textbf{35.90}{\scriptsize$\pm$0.09} & \textbf{17.44}{\scriptsize$\pm$0.90} & \textbf{27.21} & -- \\
\rowcolor{taurow}
RKL$+\tau$           & 27.61{\scriptsize$\pm$0.89} & 17.44{\scriptsize$\pm$0.93} & 33.57{\scriptsize$\pm$0.42} & 34.37{\scriptsize$\pm$0.11} & 17.10{\scriptsize$\pm$0.53} & 26.02 & $-1.19$ \\
\midrule
AKL                  & 27.88{\scriptsize$\pm$0.11} & 17.37{\scriptsize$\pm$0.63} & 31.24{\scriptsize$\pm$0.20} & 32.27{\scriptsize$\pm$0.07} & 16.82{\scriptsize$\pm$0.49} & 25.12 & -- \\
\rowcolor{taurow}
AKL$+\tau$           & 28.31{\scriptsize$\pm$0.61} & 18.39{\scriptsize$\pm$1.10} & 34.31{\scriptsize$\pm$0.48} & 34.83{\scriptsize$\pm$0.15} & 17.00{\scriptsize$\pm$0.40} & 26.57 & $+1.45$ \\
\midrule
SFKL                 & 26.98{\scriptsize$\pm$0.18} & 17.60{\scriptsize$\pm$1.36} & 31.62{\scriptsize$\pm$0.20} & 33.06{\scriptsize$\pm$0.22} & 16.53{\scriptsize$\pm$0.47} & 25.16 & -- \\
\rowcolor{taurow}
SFKL$+\tau$          & 28.08{\scriptsize$\pm$0.17} & 18.14{\scriptsize$\pm$0.79} & 34.41{\scriptsize$\pm$0.24} & 34.68{\scriptsize$\pm$0.10} & 17.11{\scriptsize$\pm$0.32} & 26.48 & $+1.32$ \\
\midrule
SRKL                 & \textbf{28.93}{\scriptsize$\pm$0.17} & \textbf{19.31}{\scriptsize$\pm$0.73} & 32.88{\scriptsize$\pm$0.27} & 34.28{\scriptsize$\pm$0.12} & 16.94{\scriptsize$\pm$0.58} & 26.47 & -- \\
\rowcolor{taurow}
SRKL$+\tau$          & 27.80{\scriptsize$\pm$0.50} & 18.07{\scriptsize$\pm$0.53} & 34.56{\scriptsize$\pm$0.31} & 34.80{\scriptsize$\pm$0.09} & 16.94{\scriptsize$\pm$0.37} & 26.38 & $-0.09$ \\
\bottomrule
\end{tabular}}
\end{table*}

\begin{table*}[t]
\centering
\scriptsize
\caption{Rouge-L ($\uparrow$) for distillation from Qwen3-4B to Qwen3-1.7B using FKL, RKL, AKL, SFKL, and SRKL. $\Delta$ denotes the change in average Rouge-L from applying temperature tuning. Bold: best value per column.}
\label{tab:temperature_1.7B4B}
\resizebox{0.85\linewidth}{!}{%
\begin{tabular}{lccccccc}
\toprule
Loss & Dolly & Self-Inst & Super-NI & UnNI & Vicuna & Avg. & $\Delta$ \\
\midrule
FKL                  & 26.56{\scriptsize$\pm$0.46} & 19.17{\scriptsize$\pm$0.55} & 34.83{\scriptsize$\pm$0.73} & 30.56{\scriptsize$\pm$0.13} & 16.69{\scriptsize$\pm$0.49} & 25.57 & -- \\
\rowcolor{taurow}
FKL$+\tau$           & 27.75{\scriptsize$\pm$0.71} & 19.47{\scriptsize$\pm$0.66} & 35.36{\scriptsize$\pm$0.29} & 32.73{\scriptsize$\pm$0.20} & 17.91{\scriptsize$\pm$0.50} & 26.65 & $+1.08$ \\
\midrule
RKL                  & \textbf{29.39}{\scriptsize$\pm$0.50} & 20.66{\scriptsize$\pm$1.16} & \textbf{39.39}{\scriptsize$\pm$0.03} & \textbf{36.48}{\scriptsize$\pm$0.20} & \textbf{19.12}{\scriptsize$\pm$0.51} & \textbf{28.99} & -- \\
\rowcolor{taurow}
RKL$+\tau$           & 29.02{\scriptsize$\pm$0.30} & 19.11{\scriptsize$\pm$0.43} & 35.97{\scriptsize$\pm$0.51} & 34.44{\scriptsize$\pm$0.09} & 18.88{\scriptsize$\pm$0.74} & 27.49 & $-1.50$ \\
\midrule
AKL                  & 26.94{\scriptsize$\pm$0.29} & 18.75{\scriptsize$\pm$0.59} & 36.86{\scriptsize$\pm$0.43} & 33.03{\scriptsize$\pm$0.08} & 17.74{\scriptsize$\pm$0.35} & 26.67 & -- \\
\rowcolor{taurow}
AKL$+\tau$           & 28.04{\scriptsize$\pm$0.45} & 19.37{\scriptsize$\pm$0.30} & 35.39{\scriptsize$\pm$0.26} & 33.42{\scriptsize$\pm$0.14} & 18.09{\scriptsize$\pm$0.69} & 26.86 & $+0.19$ \\
\midrule
SFKL                 & 27.87{\scriptsize$\pm$0.52} & \textbf{20.70}{\scriptsize$\pm$1.59} & 36.88{\scriptsize$\pm$0.25} & 33.44{\scriptsize$\pm$0.19} & 17.28{\scriptsize$\pm$0.13} & 27.23 & -- \\
\rowcolor{taurow}
SFKL$+\tau$          & 28.20{\scriptsize$\pm$0.36} & 20.22{\scriptsize$\pm$0.10} & 35.77{\scriptsize$\pm$0.34} & 33.63{\scriptsize$\pm$0.03} & 18.45{\scriptsize$\pm$0.39} & 27.25 & $+0.02$ \\
\midrule
SRKL                 & 28.91{\scriptsize$\pm$0.29} & 20.46{\scriptsize$\pm$1.12} & 39.12{\scriptsize$\pm$0.12} & 36.11{\scriptsize$\pm$0.37} & 18.34{\scriptsize$\pm$0.86} & 28.59 & -- \\
\rowcolor{taurow}
SRKL$+\tau$          & 29.01{\scriptsize$\pm$0.04} & 20.47{\scriptsize$\pm$1.09} & 36.64{\scriptsize$\pm$0.12} & 34.44{\scriptsize$\pm$0.09} & 18.57{\scriptsize$\pm$0.49} & 27.83 & $-0.76$ \\
\bottomrule
\end{tabular}}
\end{table*}

\section{More Results}
\label{app:more_exp}

\boldhdr{Additional benchmarks support the main finding}
Table~\ref{tab:extra} reports results on Self-Instruct and Vicuna. 
Temperature improves every objective across both GPT-2 Base and OPT-1.3B, confirming that its benefit generalizes beyond the main benchmarks. 
For GPT-2 Base, AKL and FKL obtain the largest gains, improving by $+2.94$ and $+2.33$ average ROUGE-L, respectively, and both surpass DRKL. 
For OPT-1.3B, all objectives improve by $+1.29$ to $+2.06$, with SRKL+$\tau$ achieving the best average score of $16.73$. 
These results show that KD temperature is a broadly useful factor for KL-based distillation and can alter the ranking among objectives.

\boldhdr{Modern model families} 
To further strengthen the empirical evidence on modern model families, we additionally evaluate two teacher–student pairs: 
Llama 3.2-3B $\to$ Llama 3.2-1B and Qwen3-4B $\to$ Qwen3-1.7B. 
The results are reported in Tables~\ref{tab:temperature_1B3B} and~\ref{tab:temperature_1.7B4B}.
Both modern models exhibit the same qualitative pattern observed for GPT-2 and OPT: FKL-based objectives benefits more from temperature tuning than RKL-based objectives. On Llama3.2, FKL, AKL, and SFKL improve by $+3.16$, $+1.45$, and $+1.32$ average Rouge-L, whereas RKL and SRKL decrease by $-1.19$ and $-0.09$. On Qwen3, FKL, AKL, and SFKL gain by $+1.08$, $+0.19$, and $+0.02$, while RKL and SRKL drop by $-1.50$ and $-0.76$. These results indicate that the asymmetric effect of temperature is not specific to any particular LLM architecture.

\begin{table*}[ht]
\centering
\caption{Qualitative examples from the Unnatural Instructions~\citep{honovich2023unnatural} benchmark. 
The OPT-1.3B student is distilled from the OPT-6.7B teacher. 
With KD temperature, all distillation objectives generate responses that better satisfy the given instructions compared with their $\tau=1$ counterparts.}
\vspace{5pt}
\resizebox{\textwidth}{!}{
\addtolength{\tabcolsep}{2.5pt}
\begin{tabular}{m{3cm}m{20cm}}
\toprule[0.1em]
        \multicolumn{2}{c}{\textbf{\textit{Case study \# 1}}} \\ \midrule[0.1em]
        Instruction & Given a paragraph, write down all the proper nouns in the order they occur. A proper noun is a specific name for a person, place, or thing and is always capitalized. \\ \midrule
        Input & Passage: Toby thought that Mr. Krebs was being unfair when he gave out detentions for talking during class, but Mrs. Dvorak said that it was part of his job as a teacher.\\ \midrule
        \rowcolor{green!10} Ground-truth & Toby, Mr. Krebs, Mrs. Dvorak \\
        \midrule
        \midrule
        FKL    & Toby, Ms. Dvorak, and Mr. Krebs are all proper nouns. \\
        \midrule
        \rowcolor{blue!10} FKL + $\tau$   & Toby, Mr. Krebs, Mrs. Dvorak \\
        \midrule
        RKL   & Toby, Mrs. Dvorak, Mr. Krebs, Class \\
        \midrule
        \rowcolor{blue!10} RKL + $\tau$   & Toby, Mr. Krebs, Mrs. Dvorak \\
        \midrule
        Sym-KL    & Toby, Mr. Krebs, Mrs. Dvorak \\
        \midrule
        \rowcolor{blue!10} Sym-KL + $\tau$   & Toby, Mr. Krebs, Mrs. Dvorak \\
        \midrule
        JS  & Toby, Mrs. Dvorak, and Mr. Krebs Waited for the perfect opportunity to strike back at Mr. Krabs. When he was delivering the class talk, he stopped cold and asked them to exchange their places. Toby assumed that Mr. Krabs was being unfair because he was giving out detentions for speaking during class. Mrs. Dvorak explained that it was part of her job as a teacher to make sure that class was taught properly.  \\ 
        \midrule
        \rowcolor{blue!10} JS + $\tau$   & Toby, Mr. Krebs, Mrs. Dvorak \\
        \midrule
        SFKL    & Toby, Mrs. Dvorak \\
        \midrule
        \rowcolor{blue!10} SFKL + $\tau$   & Toby, Mr. Krebs, Mrs. Dvorak \\
        \midrule
        SRKL    & Toby, Mrs. Dvorak, and Mr. Krebs all agreed that Mr. Krebs should have given out detentions for discussing classes during class, but Mrs., Ms., and Mr. Krebs disagreed.  Regarding Mr. Krebs' position on detentions, Mrs. Dvorac stated that it was his job to teach, and therefore he should teach what is required of him.  Regarding Mr. Bacon's position on speech, Mrs. Dvorack said that he was a teacher, and therefore he should impart wisdom and knowledge to his students.  Regarding Mr. Schmidt's position on language, Mrs. Dvoraken stated that he was a professor, and therefore he should use his knowledge to teach what is required of himself. \\
        \midrule
        \rowcolor{blue!10} SRKL + $\tau$   & Toby, Mr. Krebs, Mrs. Dvorak \\
        \midrule
        AKL    & Toby, Mrs. Dvorak, and Mr. Krebs are all proper nouns. \\
        \midrule
        \rowcolor{blue!10} AKL + $\tau$   & Toby, Mr. Krebs, Mrs. Dvorak \\
        \midrule
        AB    & 1. Toby thought that Mr. Krebs was being unfair when He gave out detentions For talking during class, but Mr. Dvorak said it was part of his Job as a teacher. <n>2. Mrs. Dvorak added that it was his job to teach, and he should know what was expected of him. <n>3. Toby agreed, and added that he would remember this when he got home. \\
        \midrule
        DRKL    & Toby, Mr. Krebs, Mrs. Dvorak \\
        
    \bottomrule[0.1em]
\end{tabular}
}\label{table:case_study_unnatural_appendix_1}
\end{table*}

\begin{table*}[ht]
\centering
\caption{Qualitative examples from the Unnatural Instructions~\citep{honovich2023unnatural} benchmark. 
The OPT-1.3B student is distilled from the OPT-6.7B teacher. 
With KD temperature, all distillation objectives generate responses that better satisfy the given instructions compared with their $\tau=1$ counterparts.}
\vspace{5pt}
\resizebox{\textwidth}{!}{
\addtolength{\tabcolsep}{2.5pt}
\begin{tabular}{m{3cm}m{20cm}}
\toprule[0.1em]
        \multicolumn{2}{c}{\textbf{\textit{Case study \# 2}}} \\ \midrule[0.1em]
        Instruction & You will be given a piece of text that is missing important information, and a list of words. Your job is to fill in the blank with one of the words from the list so that the sentence/paragraph makes grammatical sense. \\ \midrule
        Input & I asked if he wanted to go for coffee, but he said \_\_\_\_\_.\texttt{<n>}Word options: ['no', 'not', 'didnt'] \\ \midrule
        \rowcolor{green!10} Ground-truth & I asked if he wanted to go for coffee, but he said no. \\
        \midrule
        \midrule
        FKL    & "I asked if he wanted to going for coffee, but he says \_\_\_\_."\texttt{<n><n>}Thus, "I asked if he wanted go for coffee, but [name] said \_\_\_\_." This constitutes fixing the grammatical error in the sentence "I asked if he wanting to go for coffee, \_\_\_\_. \\
        \midrule
        \rowcolor{blue!10} FKL + $\tau$   & I asked if he wanted to went for coffee, but he says he doesn't want to. \\
        \midrule
        RKL   & I asked if he wanted to going for coffee, but he says \_\_\_\_\_. \\
        \midrule
        \rowcolor{blue!10} RKL + $\tau$   & I asked if he wanted to take a coffee, but he said 'not' \\
        \midrule
        Sym-KL    & word options: 'no', 'not','didnt' \\
        \midrule
        \rowcolor{blue!10} Sym-KL + $\tau$   & I asked if he wanted to Go for coffee, but he says he doesn't want to. \\
        \midrule
        JS  & He would rather go for a walk.  \\ \midrule
        \midrule
        \rowcolor{blue!10} JS + $\tau$   & He wanted to go for coffee but he said he didn't want to. \\
        SFKL    & The sentence is grammatically incorrect because it should read as:\texttt{<n>}I asked if he wanted to take a coffee break, but he said \_\_\_\_.\texttt{<n>}The correct sentence is:\texttt{<n>}I asked if He wanted to take a coffee walk, but he said \_\_\_\_\_\_\_.\texttt{<n>}The grammatical problem with the statement is that it is missing the crucial phrase - hence it is incorrect grammatically. \\
        \midrule
        \rowcolor{blue!10} SFKL + $\tau$   & He asked if he wanted to Go for coffee, but he told me that he didn't want to go for coffee. \\
        \midrule
        SRKL    & Because you provided no information in the sentence, and only provided a list of words, the context is deeply misunderstood, and your answer is incorrect. You should have provided more information in the text to clarify the context. \\
        \midrule
        \rowcolor{blue!10} SRKL + $\tau$   & He wants to go for coffee, but he says he doesn't want to go. \\
        \midrule
        AKL    & Iamar did not want to go for coffee; he said [."not", "didnt", "not"]. \\
        \midrule
        \rowcolor{blue!10} AKL + $\tau$   & He wanted to go for coffee but he said he didn't want to. \\
        \midrule
        AB    & These statements are correct, but I would also recommend using the word ``to'' consistently in this case. 
        In the first statement, ``The glass fell three stories'' is a complete statement, while ``It broke into pieces'' is also a complete statement. 
        In the second statement, ``A piece hit John on the head'' is a correct statement, whereas ``A piece hit John onthe head'' and ``A piece hit Johnon thehead'' contain spacing errors and should be corrected. \\
        \midrule
        DRKL  & The correct answer is 'no'. Unless he specifically asks for a coffee, he doesn't want to go for coffee. \\
        \bottomrule
        
    \bottomrule[0.1em]
\end{tabular}
}\label{table:case_study_unnatural_appendix_2}
\end{table*}

\section{Case Studies}
\label{app:case_studies}

Tables~\ref{table:case_study_unnatural_appendix_1} and~\ref{table:case_study_unnatural_appendix_2} show qualitative examples from the Unnatural Instructions benchmark~\citep{honovich2023unnatural}. 
The OPT-1.3B student is distilled from the OPT-6.7B teacher under different distillation objectives, with and without KD temperature. 
Across both cases, applying temperature consistently improves instruction following and output format compared with the corresponding $\tau=1$ objectives.

In Case study \#1, the task requires extracting proper nouns in their original order. 
Without temperature, several objectives omit entities, add incorrect nouns, or generate unrelated continuations: RKL includes ``Class,'' SFKL omits ``Mr. Krebs,'' and JS/SRKL produce long hallucinated passages. 
With temperature, all KL-based objectives produce the exact ground-truth output: ``Toby, Mr. Krebs, Mrs. Dvorak,'' suggesting better preservation of fine-grained constraints such as output format and entity selection.

Case study \#2 is more challenging, but shows a similar trend. 
At $\tau=1$, several objectives misunderstand the blank-filling instruction: FKL and RKL copy or corrupt the incomplete sentence, Sym-KL repeats the word options, and SRKL critiques the prompt instead of answering. 
With temperature, the outputs become more aligned with the intended completion, often expressing that the person did not want coffee. 
Although some responses do not exactly match the ground-truth word ``no,'' they are substantially closer than their $\tau=1$ counterparts. 
Overall, these examples show that KD temperature improves not only ROUGE-L scores, but also the student's ability to follow instructions and produce task-appropriate responses.

\section{Use of AI Assistants}
\label{appx:llms}

In preparing this paper, we made limited use of AI Assistants such as ChatGPT only to help refine wording, correct grammar, and enhance clarity. All core research contributions, including conceptual development, theoretical analysis, experimental evaluations, and paper organization and composition, were conducted entirely by the authors.

\end{document}